\documentclass[final]{colt2023}

\title[The Sample Complexity of Multi-Distribution Learning]{Open Problem: The Sample Complexity of Multi-Distribution Learning for VC Classes}
\usepackage{times}
\usepackage[numbers]{natbib}
\usepackage{mathtools}
\usepackage{thm-restate}
\usepackage{algorithm}
\usepackage{algorithmic}
\usepackage[suppress]{color-edits}
\usepackage{booktabs}
\usepackage{multirow}
\addauthor{ez}{blue}
\addauthor{nh}{purple}

\newcommand{\rloss}{q}
\newcommand{\ncost}{\hat{\cost}}

\newtheorem{problem}{Problem}

\newtheorem{fact}[theorem]{Fact}

\renewcommand{\epsilon}{\varepsilon}
\renewcommand{\hat}{\widehat}
\renewcommand{\tilde}{\widetilde}

\newcommand{\red}[1]{{\color{red}#1}}

\newcommand{\para}[1]{\left(#1\right)}
\newcommand{\paraflat}[1]{(#1)}
\newcommand{\parantheses}[1]{\left(#1\right)}

\newcommand{\curlybrackets}[1]{\left\{#1\right\}}

\newcommand{\bset}[1]{\curlybrackets{#1}}
\newcommand{\bsetflat}[1]{\{#1\}}

\newcommand{\abs}[1]{\left|#1\right|}

\newcommand{\setsize}[1]{\left| #1 \right|}

\DeclareMathOperator*{\Exp}{\mathbb{E}}

\newcommand{\EEs}[2]{\Exp_{#1}\left[#2\right]}

\newcommand{\EEsc}[3]{\Exp_{#1}\left[#2 \mid #3\right]}

\renewcommand{\cite}[1]{\citep{#1}}

\newcommand{\ceil}[1]{\left\lceil#1\right\rceil}

\newcommand{\bigO}[1]{O\parantheses{#1}}

\newcommand{\bigOtildesmol}[1]{\tilde{O}(#1)}

\newcommand{\reals}{\mathbb{R}}

\newcommand{\integers}{\mathbb{Z}}

\newcommand{\naturals}{\mathbb{N}}

\newcommand{\simplex}{\Delta}
\newcommand{\simiid}{\mathrel{\stackrel{\makebox[0pt]{\mbox{\normalfont\tiny i.i.d.}}}{\sim}}}

\DeclareMathOperator*{\argmax}{arg\,max}
\DeclareMathOperator*{\argmin}{arg\,min}

\newcommand{\assignequals}{\coloneqq}

\newcommand{\asseq}{\coloneqq}

\newcommand{\features}{\mathcal{X}}

\newcommand{\labels}{\mathcal{Y}}

\newcommand{\oracle}{\text{EX}}

\newcommand{\vcd}{d}

\newcommand{\err}{\cL}
\newcommand{\risk}{\cL}
\newcommand{\loss}{\ell}

\newcommand{\opt}{\mathrm{OPT}}
\newcommand{\regret}{\mathrm{Reg}}

\newcommand{\hyp}{h}
\newcommand{\rhyp}{p}
\newcommand{\hypothesisspace}{\cH}
\newcommand{\hyps}{\cH}

\newcommand{\dist}{D}

\newcommand{\distset}{\cD}
\newcommand{\dists}{\cD}
\newcommand{\distributionspace}{\cD}

\newcommand{\tsv}[2]{#1\vphantom{#1}^{\parantheses{#2}}}

\newcommand{\cost}{{c}}

\newcommand{\action}{a}

\newcommand{\actionset}{A}
\newcommand{\actions}{A}

\newcommand{\cA}{\mathcal{A}}

\newcommand{\cD}{\mathcal{D}}

\newcommand{\cH}{\mathcal{H}}

\newcommand{\cL}{\mathcal{L}}

\newcommand{\cP}{\mathcal{P}}

\newcommand{\cZ}{\mathcal{Z}}

\newcommand{\bX}{{\mathbf{X}}}

\coltauthor{%
 \Name{Pranjal Awasthi} \Email{pranjalawasthi@google.com}\\
 \addr Google Research, Mountain View, CA, USA
 \AND
 \Name{Nika Haghtalab} \Email{nika@berkeley.edu}\\
 \addr University of California, Berkeley, CA, USA
 \AND
 \Name{Eric Zhao} \Email{eric.zh@berkeley.edu}\\
 \addr University of California, Berkeley, CA, USA
}

\begin{document}

\maketitle

\begin{abstract}
	Multi-distribution learning is a natural generalization of PAC learning to settings with multiple data distributions.
	There remains a significant gap between the known upper and lower bounds for PAC-learnable classes.
	In particular, though we understand the sample complexity of learning a VC dimension $d$ class on $k$ distributions to be $O(\epsilon^{-2} \ln(k) (d + k) + \min \bset{\epsilon^{-1} d k,  \epsilon^{-4} \ln(k) d })$, the best lower bound is ${\Omega}(\epsilon^{-2}(d + k \ln(k)))$.
	We discuss recent progress on this problem and some hurdles that are fundamental to the use of game dynamics in statistical learning. 
	
\end{abstract}


\section{Introduction}
The pervasive need for robustness, fairness, and multi-agent welfare in  learning processes has led to the development of learning paradigms whose performance hold under multiple distributions and scenarios.
\emph{Multi-distribution learning}, or MDL, is a setting introduced by~\cite{haghtalabOnDemandSamplingLearning2022} to address these needs and unify several existing frameworks and applications, such as notions of \emph{min-max} fairness \cite{mohri_agnostic_2019,Abernethy2022}, \emph{group distributionally robust} optimization \cite{sagawa_distributionally_2020}, and collaborative learning \cite{blumCollaborativePACLearning2017}.
MDL is a generalization of the agnostic learning paradigms~\citep{valiant_theory_1984,blumer1989learnability} to multiple data distributions. In this setting, given a set of distributions $\dists = \bset{\dist_1, \dots, \dist_k}$ supported on $\features \times \labels$, loss function $\loss$, and a hypothesis class $\hyps$, %
the goal of MDL is to find a (possibly randomized) hypothesis $\hyp$ where
\begin{align}
	\label{eq:optimal}
	\smash{\max_{\dist \in \dists} \risk_{\dist}(\hyp) \leq \epsilon + \min_{\hyp^* \in \hyps} \max_{\dist \in \dists} \risk_{\dist}(\hyp^*),\; \text{where}\; \risk_\dist(\hyp) \asseq \EEs{(x,y) \sim \dist}{\loss(\hyp, (x, y))}.}
\end{align}
Such an $\hyp$ is called an \emph{$\epsilon$-optimal solution} to the MDL problem $(\dists, \hyps)$ and we denote
$\opt \asseq \min_{\hyp^* \in \hyps} \max_{\dist \in \dists} \risk_{\dist}(\hyp^*)$.
Our open problem concerns the sample complexity of MDL.

\paragraph{Problem Statement.}
Consider an example oracle 
$\oracle_i$ for each distribution $\dist_i \in \dists$, which once queried returns an independent sample $(x,y)\sim \dist_i$.
The optimal sample complexity of MDL is the smallest total number of queries issued to examples oracles, in a possibly adaptive fashion, that is sufficient for learning an $\epsilon$-optimal solution.
Formally, 
a multi-distribution learning algorithm at each iteration $t = 1,2, \dots$, chooses an index $\tsv{i}{t} \in [k]$, queries $\oracle_{\tsv{i}{t}}$ to sample an instance $(\tsv{x}{t}, \tsv{y}{t})$ and, upon termination, returns a (possibly randomized) solution $\hyp$.
We use the shorthands $\tsv{z}{t} = (\tsv{x}{t}, \tsv{y}{t}, \tsv{i}{t})$, $\cZ = \features \times \labels \times [k]$, and $\cZ^*$ to denote a sequence $\tsv{z}{1},\tsv{z}{2}, \dots$ of any size.
\begin{definition}[Multi-Distribution Learnability]
We say a hypothesis class $\hyps$ is multi-distribution learnable with sample complexity $m_\hyps: (0, 1)^2 \times \naturals \to \naturals$ if there exists functions $\cA_s: \cZ^* \to [k]$ and $\cA_\hyp: \cZ^* \to \simplex(\labels)^\features$ where the following holds: for every $(\epsilon, \delta) \in (0, 1)$, $k \in \naturals$, and set of $k$ distributions $\dists$ over $\features \times \labels$, by letting $\tsv{i}{t} = \cA_s(\tsv{z}{1}, \dots, \tsv{z}{t-1})$ for $t \in [m_\hyps(\epsilon, \delta, k)]$, with probability at least $1 - \delta$, the solution $\hyp = \cA_\hyp(\tsv{z}{1}, \dots, \tsv{z}{m})$ is $\epsilon$-optimal, i.e., satisfying \eqref{eq:optimal}.
\end{definition}
\begin{problem}
\label{prob:main}
What is the optimal sample complexity of MDL? 
Are hypothesis classes  $\hyps$ with VC dimension $d$  multi-distribution learnable with a sample complexity of $O\left( \epsilon^{-2}(\ln(k) d + k \ln(k / \delta)\right))$?

\end{problem}

Recalling that the sample complexity of agnostic learning is $m_\hyps(\epsilon, \delta, 1) \in \Theta(\epsilon^{-2}(d + \ln(1/ \delta)))$ \cite{shai}, one hopes to avoid paying the $\Omega\paraflat{k \cdot m_\hyps(\epsilon, \delta/k, 1)}$ samples necessary to independently learn each of the $k$ data distributions. This is why our conjectured sample complexity avoids a dependence on $dk$ and has an optimal $\epsilon^{-2}$ dependence.
Existing results, however, have fallen short of meeting both of these requirements and traded off lack of dependence on $dk$ with the optimal dependence on $\epsilon$, as shown in rows 1 and 2 of Table~\ref{tab:bounds}.
On the other hand, the optimal sample complexity of MDL has been rightly characterized for finite hypothesis classes in row 3 (and more generally those of finite Littlestone dimension or Bregman diameter~\citep{haghtalabOnDemandSamplingLearning2022}) and obtains optimal $\epsilon^{-2}\ln(|\hyps|)$ dependence.
The best lower bound, row 4, leaves a logarithmic gap with the conjectured upper bound.
Near-optimal bounds are known for \emph{realizable} settings where $\opt \mkern-4mu = \mkern-4mu 0$ (row 5) and \emph{personalized} settings where one can produce a different hypothesis for each distribution (row 6).
\begin{table}[htbp]
	\centering
	\caption{Best known bounds on the sample complexity of MDL for hypothesis classes with VC dimension $d$. $\tilde{O}$ hides double-log factors and an additive factor of $\epsilon^{-2} k \ln(k/\delta)$.}
	\label{tab:bounds}
	\begin{tabular}{llll}
		\toprule
		&\textbf{Bound}                                                                                           & \textbf{Assumption} & \textbf{Citation}                                 \\
		\midrule
		1.&$\tilde{O}(\epsilon^{-2} \ln(k) d 
 + \epsilon^{-1} \red{dk} \log(d/\epsilon)) $ & N/A                 & \cite{haghtalabOnDemandSamplingLearning2022}     \\
		2.&$\tilde{O}(\red{\epsilon^{-4}} \ln(k)(d + \ln(1/\delta\epsilon))$             & N/A                 & (Theorem~\ref{theorem:diff})                        \\
		3.&$\tilde{O}(\epsilon^{-2}\ln(\red{\setsize{\hyps}}))$                                               & N/A                 & \cite{haghtalabOnDemandSamplingLearning2022}     \\
		4.&$\Omega(\epsilon^{-2}(d + k \ln(\min\bset{d, k}/\delta)))$                                               & N/A                 & \cite{haghtalabOnDemandSamplingLearning2022}     \\
		\midrule
		5.&	${O}(\ln(k) \epsilon^{-1} (d \ln(1/\epsilon)  + k  \ln(k / \delta) ))$ & $\opt = 0$          & \cite{chenTightBoundsCollaborative2018,nguyenImprovedAlgorithmsCollaborative2018} \\
	6.&	$\tilde{O}(\ln(k) \epsilon^{-2} (d \ln(d/\epsilon)  + k  \ln(k / \delta) ))$                              & Personalized       & (Theorem~\ref{theorem:personal})                        \\
		\bottomrule
	\end{tabular}
\end{table}

\paragraph{Broad Applications.}
One of the motivating application of MDL is \emph{collaborative learning}, where multiple stakeholders (representing $\dist_i$) collaborate in training a model that provides high performance for each stakeholder~\cite{blumCollaborativePACLearning2017,nguyenImprovedAlgorithmsCollaborative2018,chenTightBoundsCollaborative2018,blum_one_2021}.
The sample complexity of MDL thus quantifies the value of collaboration in learning: whereas our conjectured upper bound would imply that collaboration reduces the amount of data needed by a $\ln(k) / k$ factor, existing bounds only imply a $\min \bset{\ln(k) / k \epsilon^2, \epsilon}$ factor reduction.

Another application of MDL is to Group \emph{distributionally robust optimization} (DRO) which concerns learning a model with performance guarantees for  many deployment environments \cite{sagawa_distributionally_2020,sagawa_investigation_2020}.
MDL sample complexity bounds quantify the cost of obtaining this robustness, a question of growing interest and which has been studied in terms of finite-sum convergence \cite{carmonhausler,asi2021} and sample complexity \cite{haghtalabOnDemandSamplingLearning2022}.
Our conjectured upper bound would extend these favorable results to VC classes by only increasing the sample complexity logarithmically.

MDL also captures notions of min-max fairness in learning, which concerns prioritizing the well-being of the worst-off subgroup and has applications in federated learning \cite{mohri_agnostic_2019} and equity \cite{Abernethy2022}.
Min-max fair learning has mainly been studied in settings with presampled datasets, where an inevitable sample complexity lower bound of $\Omega(dk/\epsilon^2)$ arises as one cannot adaptively choose distributions to sample from.
The sample complexity of MDL thus captures how min-max fairness can be attained at less cost by adapting one's data collection strategy on the fly.

\section{Overview of Current Approaches}
Multi-distribution learning can be formulated as the zero-sum game between a ``learner'' who chooses hypotheses $\hyp \in \hyps$ and an ``adversary'' whose chooses indices $i \in [k]$, with the payoff function $\risk_{D_i}(h)$.
Importantly, for any mixed-strategy $\epsilon$-min-max equilibrium $(\rhyp, \rloss) \in \simplex(\hyps) \times \simplex_k$, the randomized map $\rhyp$ is a  $2 \epsilon$-optimal solution.
All existing multi-distribution learning algorithms can be expressed as finding a $\epsilon$-equilibrium using no-regret dynamics (see \cite{haghtalabOnDemandSamplingLearning2022} for an overview).

\paragraph{Game dynamics.}
Formally, a game dynamic is a $T$-iteration process where, at each $t \in [T]$, a learner chooses hypothesis $\smash{\tsv{\hyp}{t} \in \hyps}$ with a no-regret algorithm and an adversary chooses a distribution $\smash{\tsv{i}{t} \in [k]}$ with a (semi-)bandit algorithm.
The learner estimates its current cost function $\smash{\hyp \mapsto \risk_{\dist_{\tsv{i}{t}}}(\hyp)}$ by sampling $N_{\mathrm{learn}}$ datapoints from $\smash{\oracle_{\tsv{i}{t}}}$, while the adversary estimates its cost function $\smash{i \mapsto - \risk_{\dist_{i}}(\tsv{\hyp}{t})}$ by, for $N_{\mathrm{adv}}$ choices of $i \in [k]$, sampling a datapoint from each $\oracle_i$.
The random mapping $\rhyp$ where $\rhyp(x) = \text{Uniform}(\tsv{\hyp}{1}(x), \dots, \tsv{\hyp}{t}(x))$ is a $2 \epsilon$-optimal solution.

\paragraph{Different instantiations.}
Every result in Table~\ref{tab:bounds} can be obtained by instantiating this game dynamics template.
Row 3 can be obtained 
by setting $N_{\mathrm{learn}}=N_{\mathrm{adv}}=1$, $T \propto \epsilon^{-2}(\ln(\setsize{\hyps}) + k \ln(k/\delta))$, having the learner choose $\tsv{\hyp}{t}$ with Hedge and the adversary choose $\tsv{i}{t}$ with Exp3 \cite{haghtalabOnDemandSamplingLearning2022}.
Row 1 can be obtained with the same algorithm but first creating an offline $\epsilon$-covering 
 of the class $\hyps$ on each data distribution $D_i \in \dists$, using $O(d/\epsilon)$ samples per distribution.
Row 2
can be obtained by setting $N_{\mathrm{adv}}=k$, $N_{\mathrm{learn}}\propto \epsilon^{-2}(d + \ln(1/\delta\epsilon))$, $T \propto \epsilon^{-2} \ln(k/\delta)$, having the learner choose $\tsv{\hyp}{t}$ to be the (approximate) risk minimizer of the current cost function and the adversary choose $\tsv{i}{t}$ with Hedge (Theorem~\ref{theorem:diff});
in contrast to the prior upper bound, this  bound uses an algorithm that iterates fewer times but samples more at each iteration.

\paragraph{Personalization.}
We can pinpoint the challenge of negotiating trade-offs between different data distributions as the primary difficulty of handling infinite classes.
Consider the personalized setting where, during inference time, $\cA_\hyp(\tsv{z}{1}, \dots, \tsv{z}{m})$ can return a different hypothesis $h_i$ for each distribution $\dist_i$.
This assumes away the difficulty of combining hypotheses that are each near-optimal for different distributions.
The conjectured sample complexity bound of $\tilde{O}(\ln(k) \epsilon^{-2} (d \ln(d/\epsilon)  + k  \ln(k / \delta) ))$ can be obtained in the personalized setting (Row 6 of Table~\ref{tab:bounds}) by running the Row 1 algorithm $\ln(k)$ times, at each round limiting the adversary to playing within a small region of the simplex $\Delta_k$ that we can efficiently cover $\hyps$ on (Theorem~\ref{theorem:personal}).

\subsection{Existing Challenges}
\paragraph{Adaptive coverings.}
A potential approach to closing the gap with the conjectured sample complexity bound is to find a method of adaptively covering the hypothesis class $\hyps$.
Whereas Row 1 was obtained by taking a naive offline $\epsilon$-covering of $\hyps$ on all $k$ distributions, Row 2 was obtained by an algorithm that (implicitly) $\epsilon$-covers the class $\hyps$ on $O(\ln(k)\epsilon^{-2})$ adaptive choices of $D_i \in \dists$.
It is unclear whether a covering of lower resolution can be used, or if it is possible to only cover $\hyps$ on $O(\ln(k))$ choices of distributions $D_i \in \dists$.
 We also note that it is not the size of the $\epsilon$-covering of $k$ distributions, i.e., $k{\epsilon}^{-O(d)}$, that is the bottleneck, but rather the number of samples needed %
 to create such a cover. %
In contrast, the personalized algorithm decided in an online fashion what distributions need to be covered and it only covers $\hyps$ on $O(\ln(k))$ choice of (mixture) distributions from $\dists$.

\paragraph{Agnostic-to-realizable.}
Another potential tool is an agnostic-to-realizable reduction \cite{hopkins22}, since nearly-optimal sample complexity bounds are known for realizable settings where $\opt = 0$ \cite{blumCollaborativePACLearning2017,chenTightBoundsCollaborative2018,nguyenImprovedAlgorithmsCollaborative2018}.
This technique has had success in related problems, such as the closely related adversarial PAC learning problem \cite{montasserVC2019}.
Unfortunately, because multi-distribution learning involves online decision-making---determining which example oracles to call---the usual reduction of testing all possible labelings of observed datapoints is intractable.

\paragraph{Bounding regret.}
Game dynamics algorithms rely on the learner achieving a low regret on the sequence of distributions chosen by the adversary.
However, with VC classes,  even when all distributions share a Bayes classifier, an oblivious adversary can force the learner to suffer regret linear in $k$.
It is therefore necessary to reason about the adversary's behavior to bound the regret of the learner.
This is atypical; game dynamics proofs usually bound each player's regret independently.
\begin{restatable}{proposition}{difficult}
	\label{proposition:difficult}
	Consider an algorithm $\cA$ that, given distributions $D_1, \dots, D_T$, draws only $N$ datapoints in total and returns a sequence of hypotheses $\hyp_1, \dots, \hyp_k$ where each $\hyp_t$ is trained only on datapoints sampled from $D_1, \dots, D_t$.
	There exists a sequence $D_1, \dots, D_T$ with only $k$ distinct members, where $\smash{\mathbb{E}[T^{-1} \sum_{t \in [T]} \err_{D_{t}}(\hyp_t)] - \min_{h^* \in \cH} T^{-1} \sum_{t \in [T]} \err_{D_{t}}(h^*)
			\in \Omega \paraflat{\sqrt{\vcd k / N}}.}$
\end{restatable}

\section{Intermediate Open Problems}
\paragraph{Lower Bounds.}
We believe a $\ln(k) \vcd$ factor is missing from the best known sample complexity lower bound of $\Theta(\epsilon^{-2}(\vcd + k \ln(\min\bset{k, \vcd} / \delta)))$.
The absence of a $\ln(k) \vcd$ term would be significant as it would imply that, when VC dimension dominates sample complexity, handling more data distributions comes effectively for free.
Interestingly, this $\ln(k)$ factor does not appear in the upper bound when the complexity of $\hyps$ is characterized by Littlestone dimension, perhaps due to the stronger compression guarantees for online-learnable classes.
A $\ln(k) \vcd$ term would also shed light on compression schemes for VC classes \cite{littlestone1986relating}; a lower bound of $\Theta(\ln(k) \vcd + k)$ would lend evidence against the existence of $O(\text{VC}(\hyps))$-size compression schemes.
\begin{problem}
\label{problem:lowerbound}
Is the sample complexity of multi-distribution learning in $\Omega(\log(k) \vcd)$?
\end{problem}

\paragraph{Proper learning.}
All existing multi-distribution learning algorithms with fast sample complexity rates produce either a randomized hypothesis $\hyp \in \simplex(\hyps)$ or an improper hypothesis resulting from taking a majority vote.
An open question is whether improperness is necessary for fast rates.
\begin{problem}
What is the sample complexity of proper multi-distribution learning?
\end{problem}
\paragraph{Oracle-efficient learning.}
For oracle-efficient algorithms, that is an algorithm only accessing $\hyps$ through an ERM oracle \cite{dudik2020}, only the sample complexity bound from Row 2 in Table~\ref{tab:bounds} is known.
An open question is whether there exists a statistical-computational trade-off for MDL. 
\begin{problem}
What is the sample complexity of oracle-efficient multi-distribution learning?
\end{problem}

\bibliographystyle{alpha}
\newcommand{\etalchar}[1]{$^{#1}$}
\newcommand{\nips}[1]{Advances in Neural Information Processing Systems #1}

\newpage
\appendix
\newcommand{\alg}{\mathrm{Alg}}

\newpage
\appendix
\section{Omitted Proofs}
\label{appendix:br}
We first recall standard results in online learning.
We use the shorthands $\tsv{x}{1:T} \asseq \tsv{x}{1}, \dots, \tsv{x}{T}$, $\tsv{\bsetflat{f(\tsv{x}{t})}}{1:T} \asseq f(\tsv{x}{1}), \dots, f(\tsv{x}{T})$, and $f(\cdot, b) \asseq a \mapsto f(a, b)$ throughout this section.
We use $\simplex($A$)$ to denote the set of probability distributions over a set $A$, and $\simplex_d$ to denote a probability simplex in $\reals^{d-1}$.
Given a distribution $\cP \in \simplex_d$, we use $(\simplex_d)_2$ to denote the convex subset of $\simplex_d$ that is the distributions that are 2-smooth: $(\simplex_d)_2 \asseq \bset{\cP \in \simplex_d \mid \max_i \cP_i \leq 2/d}$.
\paragraph{Online learning.}
For a sequence of actions $\tsv{\action}{1}, \dots, \tsv{\action}{T} \in \actionset$ and costs $\tsv{\cost}{1}, \dots, \tsv{\cost}{T}: \actionset \to [0, 1]$, regret is defined as $\regret(\tsv{\action}{1:T}, \tsv{\cost}{1:T}) \asseq \sum_{t=1}^T \tsv{\cost}{t}(\tsv{\action}{t}) - \min_{\action^* \in \actionset} \sum_{t=1}^T \tsv{\cost}{t}(\action^*)$.
An online learning algorithm $\alg$ maps from costs $\tsv{\cost}{1:t-1}$ to a new action $\tsv{\action}{t} \in A$, where $\tsv{\action}{t} = \alg_A(\tsv{\cost}{1:t-1})$.
We recall the following online learning regret bound for probability simplices.
\begin{lemma}
    \label{lemma:hedge-basic-regret-bound}
    Let $\cA$ be a compact convex subset of $\simplex_d$ and fix a learning rate $\eta \in [0, 0.5]$.
    For any sequence of linear costs $\tsv{\cost}{1:T}$, the Hedge online learning algorithm \cite{freund_decision-theoretic_1997} chooses actions $\tsv{\action}{1:T}$, where $\tsv{\action}{t} = \mathrm{Hedge}_\cA(\tsv{\cost}{1:t-1}$, with regret $\regret(\tsv{\action}{1:T}, \tsv{\cost}{1:T}) \leq \ln(d) / \eta + \eta \min_{\action^* \in \actionset} \sum_{t=1}^T \tsv{\cost}{t}(\action^*)$.
\end{lemma}
Stochastic costs are functions $\ncost: \actions \times \cZ \to [0, 1]$ of both actions and datapoints.
We say a stochastic cost $\ncost$ is linear if $\ncost(\cdot, z)$ is linear in its first argument under any datapoint $z \in \cZ$.
We know that estimating stochastic costs with i.i.d. samples does not significantly affect the regret of an online learning algorithm.
\begin{lemma}
	\label{lemma:stochastic-approximation}
        Let $\cA$ be a compact convex subset of $\simplex_d$, $\alg$ an online learning algorithm, and $\tsv{z}{1:T} \simiid \dist$ i.i.d. samples from some data distribution $\dist$.
        For any sequence of linear stochastic costs $\tsv{\ncost}{1:T}$, applying $\alg$ to the empirical cost estimates $\tsv{\bset{\tsv{\ncost}{t}(\action, \tsv{z}{t})}}{1:T}$ such that $\tsv{\action}{t} = \alg_\cA(\tsv{\bset{\tsv{\ncost}{\tau}(\action, \tsv{z}{\tau})}}{1:t-1})$ guarantees
		\begin{align*}
			\abs{\regret(\tsv{\action}{1:T}, \tsv{\bsetflat{\mathbb{E}_{z \sim \dist}[\tsv{\ncost}{t}(\cdot, z)]}}{1:T})
				- \regret(\tsv{\action}{1:T}, \tsv{\bset{\tsv{\ncost}{\tau}(\action, \tsv{z}{\tau})}}{1:T})} \leq \bigO{\sqrt{\ln(d / \delta) T}},
		\end{align*}
		with probability at least $1 - \delta$ over the randomness of $\tsv{z}{1:T}$ \cite{nemirovski2009robust}.
\end{lemma}
We also recall the agnostic learning upper bound.
\begin{lemma}
    \label{lemma:agnostic-learning}
    Consider any stochastic cost $\ncost: \actionset \times \cZ \to [0, 1]$ and data distribution $\dist$, where $d$ is the VC dimension of $\actionset$.
    With only $O((d + \ln(1/\delta)) / \epsilon \alpha)$ samples from $\dist$, the action $a \in \actionset$ empirically minimizing  $\ncost$ is $\epsilon$-optimal with probability $1 - \delta$: $\EEs{z \sim \dist}{\ncost(a, z)} \leq \epsilon + (1 + \alpha) \min_{a^* \in \actionset} \EEs{z \sim \dist}{\ncost(a^*, z)}$ \cite{nguyenImprovedAlgorithmsCollaborative2018}.
\end{lemma}
Finally, we note that all the aforementioned results for cost sequences also apply to \emph{payoff sequences}, where the regret of actions $\tsv{\action}{1:T}$ with respect to a sequence of payoffs $\tsv{\rho}{1:T}$ is defined as $\regret_+(\tsv{\action}{1:T}, \tsv{\rho}{1:T}) \asseq \max_{\action^* \in \actionset} \sum_{t=1}^T \tsv{\rho}{t}(\action^*) - \sum_{t=1}^T \tsv{\rho}{t}(\tsv{\action}{t})$. 
Here, we use the subscript $+$ in $\regret_+$ to distinguish when regrets are stated for payoff functions.
For example, the regret bound of Hedge for payoffs can be written as follows.
\begin{lemma}
    \label{lemma:hedge-basic-regret-bound-payoffs}
    Let $\cA$ be a compact convex subset of $\simplex_d$ and fix a learning rate $\eta \in [0, 0.5]$.
    For any sequence of payoffs $\tsv{\rho}{1:T}$, the Hedge online learning algorithm \cite{freund_decision-theoretic_1997} chooses actions $\tsv{\action}{1:T}$, where $\tsv{\action}{t} = \mathrm{Hedge}_\cA(\tsv{\rho}{1:t-1})$, with regret $\regret_+(\tsv{\action}{1:T}, \tsv{\rho}{1:T}) \leq \ln(d) / \eta + \eta \max_{\action^* \in \actionset} \sum_{t=1}^T \tsv{\rho}{t}(\action^*)$.
\end{lemma}

\subsection{Proof of Theorem~\ref{theorem:diff} (Row 2 of Table~\ref{tab:bounds})}

\begin{algorithm}[h]
	\caption{Multi-Distribution Learning Algorithm.}
	\label{alg:fast}
	\begin{algorithmic}
		\STATE \textbf{Input:} Hypotheses $\hypothesisspace$, distributions $\distributionspace$, iterations $T \in \integers_+$, sub-iterations $r_1, r_2 \in \integers_+$, parameter $\alpha \in (0, 0.5)$;
		\STATE Intialize Hedge iterate $\tsv{\dist}{1}$ to be a uniform mixture of $\dists$;
		\FOR{$t = 1, 2, \dots, T$}
		\STATE Sample $r_1$ datapoints $z_1, \dots, z_{r_1}$ from $\tsv{\dist}{t}$;
		\STATE Let $\tsv{\hyp}{t} = \argmin_{\hyp \in \hyps} \sum_{i=1}^{r_1} \loss(\hyp, z_i)$ be the empirical minimizer of $\loss$;
		\STATE Sample $r_2$ datapoints $z_{\dist, (t-1) r_2 + 1}, \dots, z_{\dist, t r_2}$ from each $\dist \in \dists$;
		\STATE Use the Hedge algorithm to get the next iterate $\tsv{\dist}{t+1} \in \simplex(\dists)$, using learning rate $\alpha$ and observing the payoff $\tsv{\tilde{\rho}}{t}: \dists \to [0, 1]$ where $\tsv{\tilde{\rho}}{t}(\dist)
		= \frac{1}{r_2} \sum_{i=(t-1)r_2+1}^{t r_2} \loss(\tsv{\hyp}{t}, z_{\dist, i})$;
		\ENDFOR
		\STATE Return $\overline{h}$: a uniform distribution over $\tsv{\hyp}{1:T}$;
	\end{algorithmic}
\end{algorithm}

\begin{restatable}{theorem}{diff}
	\label{theorem:diff}
	For any $\epsilon, \alpha \in (0, 0.5)$, $\delta > 0$, $k \in \integers_+$ and binary class $\cH$, the sample complexity of MDL, $m_\cH(\epsilon + \alpha \cdot \opt, \delta, k)$, is $\bigOtildesmol{\epsilon^{-2}  \para{k \log(k / \delta) + \alpha^{-2} \log(k) (\log(1/\epsilon \delta) + \text{VC}(\hyps))}}$.
\end{restatable}
\begin{proof}
	Let $\vcd$ denote the VC dimension of $\hyps$.
        Without loss of generality, assume $\epsilon \leq \alpha$.
	Consider Algorithm~\ref{alg:fast}, fixing $T = \frac{\ln(k)}{\epsilon \alpha}$, $r_1 = C_1 \frac{d + \ln(T / \delta)}{\epsilon \alpha}$, and $r_2 = \ceil{C_2 \frac{\ln(k/\delta)}{T \epsilon^2}}$.

        \begin{fact}
            The regret of the ``adversary'' in the game dynamics induced by Algorithm~\ref{alg:fast} satisfies
            \begin{align*}
                \regret_+(\tsv{\dist}{1:T}, \tsv{\bsetflat{\risk_{(\cdot)}(\tsv{\hyp}{t})}}{1:T})
                \leq \frac{\ln(k)}{\alpha} + T \epsilon +  \alpha  \max_{D^* \in \dists} \sum_{t=1}^T \risk_{D^*}(\tsv{\hyp}{t}),
            \end{align*}
            with probability at least $1 - 2 \delta$ for some choice of universal constant $C_2$.
        \end{fact}
 \begin{proof}
        The mixture distributions $\tsv{\dist}{1:T}$ result from applying Hedge to the payoff functions $\tsv{\tilde{\rho}}{1:T}$.
        Hence, by Lemma~\ref{lemma:hedge-basic-regret-bound-payoffs},
        \begin{align*}
            \regret_+(\tsv{\dist}{1:T}, \tsv{\tilde{\rho}}{1:T}) \leq \frac{\ln(k)}{\alpha} + \alpha \max_{\dist^* \in \dists} \sum_{t=1}^T \tsv{\tilde{\rho}}{t}(\dist^*).
        \end{align*}
        To prove generalization, we will break each timestep $t$ into $r_2$ sub-timesteps.
        For every $j \in [T r_2]$, we let $\tsv{\tilde{\dist}}{j} = \tsv{\dist}{\ceil{j / r_2}}$ and define $\tsv{\tilde{\cost}}{j}$ to be the cost function $D \mapsto \frac{1}{r_2} (1 - \loss(\tsv{\hyp}{t}, z_{D, j}))$.
        We can rewrite the adversary's regret as  $  \regret_+(\tsv{\dist}{1:T}, \tsv{\tilde{\rho}}{1:T}) =   \regret(\tsv{\tilde{\dist}}{1:T r_2}, \tsv{\tilde{\cost}}{1:T r_2}) $.
        Further observe that, since $\EEs{z_{\dist, j}}{\tsv{\tilde{\cost}}{j}} = \risk_{\dist}(\tsv{\hyp}{\ceil{j / r_2}})$ for every $j \in [T r_2]$ and $\dist \in \dists$, the empirical regret is unbiased: $\regret_+(\tsv{\dist}{1:T}, \tsv{\bsetflat{\risk_{(\cdot)}(\tsv{\hyp}{t})}}{1:T})=   \regret(\tsv{\tilde{\dist}}{1:T r_2}, \tsv{\bsetflat{\EEs{z_{(\cdot), j}}{\tsv{\tilde{\cost}}{j}}}}{1:T r_2}) $.
        By Lemma~\ref{lemma:stochastic-approximation},
	\begin{align*}
		&\abs{\regret_+(\tsv{\dist}{1:T}, \tsv{\bsetflat{\risk_{(\cdot)}(\tsv{\hyp}{t})}}{1:T})
			- \regret_+(\tsv{\dist}{1:T}, \tsv{\tilde{\rho}}{1:T})} \\
   &= \abs{\regret(\tsv{\tilde{\dist}}{1:T r_2}, \tsv{\tilde{\cost}}{1:T r_2})  -
   \regret(\tsv{\tilde{\dist}}{1:T r_2}, \tsv{\bsetflat{\mathbb{E}_{z_{(\cdot), j}}[\tsv{\tilde{\cost}}{j}]}}{1:T r_2}) }
   \\
   &\leq \bigO{\sqrt{\ln(k / \delta) T / r_2}} = \bigO{T \epsilon / C_2},
	\end{align*}
        with probability at least $1 - \delta$.
        Similarly, with probability at least $1 - \delta$, $\max_{D^* \in \dists} \sum_{t=1}^T {\tsv{\tilde{\rho}}{t}}(D^*) \leq O(\frac{T \epsilon}{C_2}) + \max_{D^* \in \dists} \sum_{t=1}^T \risk_{D^*}(\tsv{\hyp}{t})$.
 A union bound yields the claimed fact.
 \end{proof}

    Next, we observe that, at each timestep $t$, $\tsv{\hyp}{t}$ is the empirical risk minimizer of $\loss$ on $C_1 (d + \log(T /\delta)) / \epsilon \alpha$ samples from $\tsv{\dist}{t}$.
    For sufficiently large $C_1$, by Lemma~\ref{lemma:agnostic-learning}, $\risk_{\tsv{\dist}{t}}(\tsv{\hyp}{t}) \leq \epsilon + (1 + \alpha) \min_{\hyp^* \in \hyps} \risk_{\tsv{\dist}{t}}(\hyp^*)$ with probability at least $1 - \delta/T$.
    By union bound, $\sum_{t=1}^T \risk_{\tsv{\dist}{t}}(\tsv{\hyp}{t}) \leq T \epsilon + T (1 + \alpha) \opt$ with probability at least $1  - \delta$.
	Putting together the regret bounds for the learner and adversary,
	\begin{align*}
		(1 - \alpha) \max_{\dist^* \in \dists} \risk_{\dist^*}(\overline{\hyp}) - 2 \epsilon
		&= (1 - \alpha) \max_{\dist^* \in \dists} \para{\frac 1 T \sum_{t=1}^T \risk_{\dist^*}(\tsv{\hyp}{t})} - 2 \epsilon \\
		&\leq \frac 1 T \sum_{t=1}^T \risk_{\tsv{\dist}{t}}(\tsv{\hyp}{t}) \\
		&\leq \epsilon + (1 + \alpha) \min_{\hyp^* \in \hyps} \frac 1 T \sum_{t=1}^T \risk_{\tsv{\dist}{t}}(\hyp^*) \\
		&\leq \epsilon + (1 + \alpha) \opt.
	\end{align*}
        We can simplify $\max_{\dist^* \in \dists} \risk_{\dist^*}(\overline{\hyp}) \leq \frac{1}{1-\alpha}(3 \epsilon + (1 + \alpha)) \opt \leq 6 \epsilon + (1 + 4 \alpha) \opt$.
        Reparameterizing $\epsilon \to \frac{1}{6} \epsilon$ and $\alpha \to \frac{1}{4} \alpha$ yields the desired claim.
	\noindent
	Our sample complexity is $(r_1 + k r_2) \times T$ and thus $\bigO{\frac{ k\ln (k/ \delta)} {\epsilon^2} + \frac{(d + \log(T /\delta)) \ln(k/\delta) }{ \epsilon^2 \alpha^2}}$.
\end{proof}
\subsection{Proof of Theorem~\ref{theorem:personal}}

\begin{algorithm}[t]
	\caption{Personalized Algorithm.}
	\label{alg:personalized}
	\begin{algorithmic}
		\STATE \textbf{Input:} Hypotheses $\hypothesisspace$, distributions $\distributionspace$;
		\STATE Initialize $\tsv{\distset}{1} = \distset$;
		\FOR{$t = 1, 2, \dots, \ceil{\log(k)}$}
		\STATE Run Algorithm~\ref{alg:mid} on $\tsv{\dists}{t}, \hyps$ to obtain $\tsv{\hyp}{t}$;
		\STATE Sample $O\para{\epsilon^{-2} (\ln(k \ln(k) / \delta))}$ datapoints $\bX_t^\dist$ from each $\dist \in \distset$;
		\STATE Let $\tsv{\distset}{t+1}$ consist of $\dist$ where $\hat{\err}_{\bX_t^\dist}(\tsv{h}{t}) > \text{Median}\para{\hat{\err}_{\bX_t^\dist}(\tsv{h}{t})}_{\dist \in \distset}$;
		\ENDFOR
		\STATE For each $\dist \in \distset$, find $t_\dist$ where $\dist \in \tsv{\distset}{t_\dist}$ but $\dist \notin \tsv{\distset}{t_\dist+1}$. Return $\para{\dist, \tsv{h}{t_\dist}}_{\dist \in \distset}$.
	\end{algorithmic}
\end{algorithm}

\begin{restatable}{theorem}{personal}
	\label{theorem:personal}
	For any $\epsilon, \delta > 0$, $k \in \integers$ and binary class $\cH$, the sample complexity $m_\cH(\epsilon, \delta, k)$ of \emph{personalized} multi-distribution learning is $\tilde{O}(\epsilon^{-2} \ln(k) (\text{VC}(\hyps) \ln(\text{VC}(\hyps)k/\epsilon)  + k  \ln(k / \delta) ))$.
\end{restatable}

\begin{algorithm}[h]
	\caption{Multi-Distribution Learning Algorithm (Mid).}
	\label{alg:mid}
	\begin{algorithmic}
		\STATE \textbf{Input:} Hypotheses $\hypothesisspace$, distributions $\distributionspace$;
		\STATE Take $\epsilon^{-1} C \para{\vcd \log(\vcd /\epsilon)+\log(1/\delta)}$ samples $x_1, \dots, x_N$ from $\text{Uniform}(\cD)$ and obtain a covering $\cH'$ of $\cH$ by projection: for every $y \in \bset{[\hyp(x_1), \dots, \hyp(x_N)] \mid \hyp \in \hyps}$, include in $\cH'$ an arbitrary choice of $h \in \cH$ such that $[h(x_1), \dots, h(x_N)] = y$; 
		\STATE Intialize Hedge iterate $\tsv{\dist}{1}$ on $(\Delta \cD)_{2}$, that is the set of 2-smooth distributions on $\cD$;
		\STATE Intialize Hedge iterate $\tsv{\hyp}{1}$ on the simplex $(\Delta \hyps')$;
		\FOR{$t = 1, 2, \dots, T$}
		\STATE Use the Hedge algorithm to get the next iterate $\tsv{\hyp}{t+1} = \text{Hedge}_{\simplex(\hyps')}(\tsv{\bset{\tsv{\hat{\cost}}{\tau}}}{1:t})$, where $\tsv{\hat{\cost}}{t}(\hyp)
		= \loss(\hyp, z)$ and $z \sim \tsv{\dist}{t}$;
		\STATE Sample a $\dist' \sim \text{Uniform}(\cD)$ and a datapoint $z \sim \dist$;
		\STATE Run Hedge algorithm to get the next iterate $\tsv{\dist}{t+1} = \text{Hedge}_{(\Delta \cD)_{2}}(\tsv{\bset{\tsv{\tilde{\cost}}{\tau}}}{1:t})$, where $\tsv{\tilde{\cost}}{t}(\dist)
		= 1[\dist' = \dist] \cdot \setsize{\cD} \cdot \Pr_{\tsv{\dist}{t}}( \dist') (1 - \loss(\tsv{\hyp}{t}, z))$;
		\ENDFOR
		\STATE Return a uniform distribution over $\tsv{\hyp}{1:T}$;
	\end{algorithmic}
\end{algorithm}

We now turn to proving this result.

\begin{lemma}
	\label{lemma:personalizedsmoothed}
	Consider the multi-distribution learning problem $(\dists, \hyps, \loss)$.
	For any $h \in \Delta \cH$, there exists a $\cD' \subseteq \cD$ where $\setsize{\cD'} \geq \setsize{\cD} / 2$ and $\max_{\dist \in (\Delta \cD)_{2}} \err_\dist(h)
	\geq \max_{\dist \in \cD'} \err_\dist(h)$.
\end{lemma}
\begin{proof}
	Fix an $h \in \Delta \cH$.
	Consider all strict minorities of $\cD$: $\text{Min} \assignequals \bset{\distset' \subseteq \distset \mid \setsize{\cD'} < \setsize{\cD} / 2}$.
	Let $\cD_{\text{MinHard}}$ denote the strict minority on which $h$ does worst, and $\cD_{\text{MinEasy}}$ denote the strict minority on which $h$ does best:
	\begin{align*}
		\cD_{\text{MinHard}} & = \argmax_{\cD^* \in \text{Min}} \frac{1}{\setsize{\cD^*}} \sum_{\dist \in \cD^*} \err(h), \quad
		\cD_{\text{MinEasy}} = \argmin_{\cD^* \in \text{Min}} \frac{1}{\setsize{\cD^*}} \sum_{\dist \in \cD^*} \err(h).
	\end{align*}
	First, we observe that $\max_{(\Delta \cD)_{2}} \err(h) \geq \EEsc{D \sim \text{Uniform}(\cD)}{\err(h)}{D \notin \cD_{\text{MinEasy}}}$, where $\text{Uniform}(\cD)$ is the uniform mixture over $\cD$.
	Second, we observe that $\EEsc{D \sim \text{Uniform}(\cD)}{\err(h)}{D \notin \cD_{\text{MinEasy}}} \geq \max_{\dist \in \cD \setminus \cD_{\text{MinHard}}} \err_\dist(h)$.
	Thus, $\cD' = \cD \setminus \cD_{\text{MinHard}}$ satisfies the desired property.
\end{proof}

\begin{lemma}
	\label{lemma:samples}
	Consider a multi-distribution learning problem $(\dists, \hyps, \loss)$.
	Algorithm~\ref{alg:mid} returns a hypothesis $\overline{\hyp}$ such that with probability $1- \delta$,
 \begin{align*}\max_{D^* \in (\Delta \cD)_2} \err_{D^*}(\overline{h}) \leq \min_{\hyp^* \in \Delta \cH} \max_{D^* \in (\Delta \cD)_2} \err_{D^*}(h^*) + \epsilon.\end{align*}
 It takes only  $\tilde{O}(\epsilon^{-2} (d \ln(dk/\epsilon)  + \ln(1 / \delta) ))$ samples.
\end{lemma}
\begin{proof}
	By construction, with probability at least $1 - \delta$, $\cH'$ is an $\epsilon$-net for $\cH$ \cite{hausslerEpsilonNets1986} under the distribution $\text{Uniform}(\dists)$.
	Consider any distribution $\cP \in \simplex(\dists)_2$.
	Because any event that happens in $\cP$ must also happen in $\text{Uniform}(\dists)$ with at least half the probability, including the event that $h(x) \neq h'(x)$,  $\cH'$ is a $2\epsilon$-net for $\cP$.
	Since the range of $\ell$ is $[0,1]$, it also follows that for any $h \in \cH$, there is an $h' \in \cH'$ such that $\smash{\abs{\err_{\cP}(h) - \err_{\cP}(h')} < 2\epsilon}$.
	We now turn to arguing that our output $\text{Uniform}(\tsv{\hyp}{1:T})$ is nearly optimal for the discretized class $\cH'$.

	We observe that $\tsv{\dist}{1:T}$ results from applying Hedge to (importance-weighted estimates of) stochastic cost functions that are bounded in $[0, 2]$.
	Moreover, the costs are bounded unbiased estimates of the true costs.
	Thus, as in our proof of Theorem~\ref{theorem:diff}, we can directly apply Hedge's regret bound (Lemma~\ref{lemma:hedge-basic-regret-bound}) and stochastic approximation (Lemma~\ref{lemma:stochastic-approximation}) to bound the adversary's regret $\regret(\tsv{\dist}{1:T}, \tsv{\bsetflat{1 - \risk_{(\cdot)}(\tsv{\hyp}{t})}}{1:T}) \leq \bigO{\sqrt{\ln(k / \delta) T}}$.
	Note that this regret is defined only over the set $\simplex(\dists)_2$.
	Therefore choosing $T = \ceil{C' \ln (k/ \delta) / \epsilon^2}$ for large $C'$ gives $\regret(\tsv{\dist}{1:T}, \tsv{\bsetflat{1 - \risk_{(\cdot)}(\tsv{\hyp}{t})}}{1:T}) \leq T\epsilon$ with probability $1 - \delta$.
	Similarly, the learner's Hedge (Lemma~\ref{lemma:hedge-basic-regret-bound}) and the stochastic approximation (Lemma~\ref{lemma:stochastic-approximation}) gives the regret bound $\regret(\tsv{\hyp}{1:T}, \tsv{\bsetflat{\risk_{\tsv{\dist}{t}}(\cdot)}}{1:T}) \leq \bigO{\sqrt{\ln(\setsize{\hyps'}) T}}$.
	Note that this regret is defined only over the set $\cH'$.
	Since $\setsize{\hyps'} \leq O((kN)^d)$ by Sauer Shelah's lemma, choosing $T \geq C \epsilon^{-2} d \ln(dk\ln(d/\epsilon\delta)/\epsilon)$ for some large constant $C$ guarantees that $\regret(\tsv{\hyp}{1:T}, \tsv{\bsetflat{\risk_{\tsv{\dist}{t}}(\cdot)}}{1:T}) \leq T \epsilon$ with probability at least $1 - \delta$.
	Putting together the regret bounds for the learner and adversary as before,
	\begin{align*}
		\max_{\dist^* \in\simplex(\dists)_2} \risk_{\dist^*}(\overline{\hyp}) - \epsilon
		\leq \frac 1 T \sum_{t=1}^T \risk_{\tsv{\dist}{t}}(\tsv{\hyp}{t}) 
		&\leq \epsilon + \min_{\hyp^* \in \hyps'} \frac 1 T \sum_{t=1}^T \risk_{\tsv{\dist}{t}}(\hyp^*) \\
		&\leq 3\epsilon + \min_{\hyp^* \in \hyps} \frac 1 T \sum_{t=1}^T \risk_{\tsv{\dist}{t}}(\hyp^*) \\
		&\leq 3\epsilon + \opt.
	\end{align*}

	\noindent
	Our sample complexity is $2 \times T$ and thus $\tilde{O}(\epsilon^{-2} (d \ln(dk/\epsilon) + \ln(1 / \delta) ))$.

\end{proof}

\begin{proof}[Proof of Theorem~\ref{theorem:personal}]
	Consider Algorithm~\ref{alg:personalized}.
    Let $\dist^*$ be the product distribution of every $\dist \in \dists$.
	Let $\vcd$ denote the VC dimension of $\hyps$.
	By Lemma~\ref{lemma:samples}, with probability at least $1 - \log(k) \delta$, for all $t \in [T]$, 
	\begin{align*}\max_{D^* \in (\Delta \tsv{\cD}{t})_{2}} \err_{D^*}(\tsv{h}{t}) \leq \min_{\hyp^* \in \Delta \cH} \max_{D^* \in (\Delta \tsv{\cD}{t})_{2}} \err_{D^*}(h^*) + \epsilon.\end{align*}
	By Lemma~\ref{lemma:personalizedsmoothed}, there exists $\cD' \subseteq \tsv{\dists}{t}$ where  $\setsize{\cD'} > \tsv{\dists}{t} / 2$ and $\max_{\dist \in (\Delta \cD)_{\tsv{\cP}{t}}} \err_\dist(\tsv{h}{t})
	\geq \max_{\dist \in \cD'} \err_\dist(\tsv{h}{t})$.
	In other words, $\opt  + \epsilon \geq \max_{\dist \in \cD'} \err_\dist(\tsv{h}{t})$.
	By uniform convergence, with probability at least $1 - \delta$, for all $t \in [T]$ and $\dist \in \tsv{\dists}{t}$, $\abs{\tsv{\hat{\err}}{t}_\dist(\tsv{\hyp}{t}) - \err_{\dist}(\tsv{\hyp}{t})} \leq \epsilon$.
	Thus, for every $\dist \in \tsv{\dist}{t} \setminus \tsv{\dist}{t+1}$, $\err_\dist(\tsv{\hyp}{t}) \leq 2 \epsilon + \opt$.
	Since the size of $\tsv{\dist}{t}$ is reduced by at least half every iteration, the algorithm terminates after $\ceil{\ln(k)}$ iterations.
	The algorithm's sample complexity comes from the samples needed for  Lemma~\ref{lemma:personalizedsmoothed} and for evaluating each $\tsv{\hyp}{t}$, and taking a union bound over all iterations.
\end{proof}

\end{document}